\documentclass{article}
\usepackage{nips14submit_e,times}

\usepackage{amsmath, amssymb, amsthm}

\usepackage{graphicx}
\usepackage{caption}
\usepackage{subcaption}

\usepackage{url}

\usepackage{verbatim}

\usepackage{./StefanTexCommandsV2}
\usepackage{willcommands}

\usepackage{hyperref}

\newcommand\eqdef{\ensuremath{\stackrel{\rm def}{=}}} 
\newcommand\I{\mathbb I}

\graphicspath{{.}}

\theoremstyle{plain}
\newtheorem{prop}{Proposition}

\newtheorem{coro}[prop]{Corollary}
\newtheorem{lemm}[prop]{Lemma}
\newtheorem{theo}[prop]{Theorem}

\theoremstyle{definition}

\theoremstyle{remark}

\usepackage{tikz}
\usetikzlibrary{fit,positioning}

\date{\today}

\newcommand{\too}{\widetilde{\oo}}
\newcommand{\Binom}{\operatorname{Binom}}
\newcommand{\Pois}{\operatorname{Poisson}}
\newcommand{\Mult}{\operatorname{Multinom}}
\newcommand{\tvarepsilon}{\tilde{\varepsilon}}
\newcommand{\ts}{\tilde{s}}
\newcommand{\tL}{\widetilde{L}}
\newcommand{\teta}{\tilde{\eta}}
\newcommand{\CBE}{C_\text{BE}}
\renewcommand{\hw}{\widehat w}
\newcommand{\taui}{\tau^{(i)}}

\newcommand{\sectshrink}{\vspace{-1mm}}
\newcommand{\figshrink}{\vspace{-4mm}}

\title{Altitude Training: \\ Strong Bounds for Single-Layer Dropout}

\makeatletter
\renewcommand*{\@fnsymbol}[1]{\ensuremath{\ifcase#1\or  \or \dagger\or \ddagger\or
   \mathsection\or \mathparagraph\or \|\or **\or \dagger\dagger
   \or \ddagger\ddagger \else\@ctrerr\fi}}
\makeatother

\newcommand{\thanksfoot}{\thanks{S. Wager and W. Fithian are supported by a B.C. and E.J. Eaves Stanford Graduate Fellowship and NSF VIGRE grant DMS--0502385 respectively.}}

\author{
Stefan Wager$^*$, \
William Fithian$^*$, \
Sida Wang$^{\dagger}$, \
\textmd{and} \ Percy Liang$^{*, \dagger}$\thanksfoot \\
Departments of Statistics$^{*}$ and Computer Science$^{\dagger}$ \\
Stanford University, Stanford, CA-94305, USA \\
\texttt{\{swager, wfithian\}@stanford.edu}, \texttt{\{sidaw, pliang\}@cs.stanford.edu}  \\
}

\nipsfinalcopy 

\begin{document}

\maketitle

\begin{abstract}
Dropout training, originally designed for deep neural networks, has been successful on high-dimensional single-layer natural language tasks. This paper proposes a theoretical explanation for this phenomenon: we show that, under a generative Poisson topic model with long documents, dropout training improves the exponent in the generalization bound for empirical risk minimization. Dropout achieves this gain much like a marathon runner who practices at altitude: once a classifier learns to perform reasonably well on training examples that have been artificially corrupted by dropout, it will do very well on the uncorrupted test set. We also show that, under similar conditions, dropout preserves the Bayes decision boundary and should therefore induce minimal bias in high dimensions.
\end{abstract}

\section{Introduction}

Dropout training \cite{hinton2012improving} is an increasingly
popular method for regularizing
learning algorithms.  Dropout is most commonly
used for regularizing deep neural networks
\cite{ba2013adaptive,goodfellow2013maxout,krizhevsky2012imagenet,wan2013regularization},
but it has also been found to improve the performance of
logistic regression and other single-layer models
for natural language tasks such as document classification and named entity recognition
\cite{wager2013dropout,wang2013fast,wang2013feature}. For
single-layer linear models, learning with dropout is equivalent
to using ``blankout noise'' \cite{van2013learning}.

The goal of this paper is to gain a better theoretical understanding of why
dropout regularization works well for natural language tasks.
We focus on the task of document classification using linear classifiers
where data comes from a generative Poisson topic model.
In this setting, dropout effectively deletes random
words from a document during training;
this corruption makes the training examples harder.
A classifier that \emph{is} able to fit the training data
will therefore receive an accuracy boost at test time on
the much easier uncorrupted examples.
An apt analogy is altitude training,
where athletes practice in more difficult situations than they compete in.
Importantly, our analysis does not rely on dropout merely creating \emph{more}
pseudo-examples for training, but rather on dropout
creating \emph{more challenging} training examples.
Somewhat paradoxically, we show that removing information from
training examples can induce a classifier that performs better at test time.

\paragraph{Main Result}
Consider training the zero-one loss empirical risk minimizer (ERM) using dropout,
where each word is independently removed with probability $\delta\in (0, \, 1)$.
For a class of Poisson generative topic
models, we show that dropout gives rise to what we call the \emph{altitude training phenomenon}:
dropout improves the excess risk of the ERM by multiplying the exponent
in its decay rate by $1/(1 - \delta)$. This improvement comes
at the cost of an additive term of $O(1/\sqrt{\lambda})$, where $\lambda$
is the average number of words per document. More formally, let $h^*$ and $\hath_\text{0}$ be the
expected and empirical risk minimizers, respectively;
let $h^*_\delta$ and $\hath_\delta$ be the corresponding quantities
for dropout training. Let $\Err(h)$ denote the error rate (on test examples) of $h$.
In Section~\ref{sec:genbound}, we show that:
\begin{equation}
\label{eq:drop_gen}
\underbrace{\Err\p{\hath_\delta} - \Err\p{h^*_\delta}}_\text{dropout excess risk} =
\too_P\p{
  {\underbrace{\p{\Err\p{\hath_\text{0}} - \Err\p{h^*}}}_\text{ERM excess risk}}
  ^{\frac{1}{1 - \delta}}
+ \frac{1}{\sqrt{\lambda}}},
\end{equation}
where $\too_P$ is a variant of big-$\oo$ in probability notation that suppresses logarithmic factors.
If $\lambda$ is large (we are classifying long documents rather than short snippets of text),
dropout considerably accelerates the decay rate of excess
risk. The bound \eqref{eq:drop_gen} holds for fixed choices of $\delta$.
The constants in the bound worsen
as $\delta$ approaches 1, and so we cannot get zero excess risk by sending $\delta$ to 1.

Our result is modular in that it converts upper bounds on the
ERM excess risk to upper bounds on the dropout excess risk.
For example, recall from classic VC theory that the ERM
excess risk is $\too_P(\sqrt{d/n})$, where $d$ is the number of features (vocabulary size)
and $n$ is the number of training examples. With dropout $\delta=0.5$, our result
\eqref{eq:drop_gen}
directly implies that the dropout excess risk is $\too_P(d/n + 1/\sqrt{\lambda})$.

The intuition behind the proof of \eqref{eq:drop_gen} is as follows:
when $\delta=0.5$,
we essentially train on half documents and test on whole documents.
By conditional independence properties of the generative topic model,
the classification score is roughly Gaussian
under a Berry-Esseen bound,
and the error rate is governed by the tails of the Gaussian.
Compared to half documents,
the coefficient of variation
of the classification score on whole documents (at test time)
is scaled down by $\sqrt{1-\delta}$ compared to half documents (at training time),
resulting in an exponential reduction in error.
The additive penalty of $1/\sqrt{\lambda}$ stems from the Berry-Esseen approximation.

Note that the bound \eqref{eq:drop_gen} only controls the dropout excess risk.
Even if dropout reduces the excess risk,
it may introduce a bias $\Err(h_\delta^*) - \Err(h^*)$,
and thus \eqref{eq:drop_gen} is useful only when this bias is small.
In Section~\ref{sec:geom},
we will show that the optimal Bayes decision boundary
is not affected by dropout under the Poisson topic model.
Bias is thus negligible  when the Bayes boundary
is close to linear.

It is instructive to compare our generalization bound to that of
Ng and Jordan \cite{ng2001discriminative}, who showed that the naive Bayes classifier
exploits a strong generative assumption---conditional
independence of the features given the label---to achieve
an excess risk of $\oo_P(\sqrt{(\log d)/n})$.
However, if the generative assumption is incorrect,
then naive Bayes can have a large bias.
Dropout enables us to cut excess risk without incurring as much bias.
In fact, naive Bayes is closely related to logistic regression trained using
an extreme form of dropout with
$\delta \rightarrow 1$. Training logistic regression with dropout rates from the range
$\delta \in (0, \, 1)$ thus gives a family of classifiers between
unregularized logistic regression and naive Bayes, allowing us to tune the
bias-variance tradeoff.

\sectshrink
\paragraph{Other perspectives on dropout}

In the general setting, dropout only improves
generalization by a \emph{multiplicative} factor.
McAllester \cite{mcallester2013pac} used the PAC-Bayes framework to prove a
generalization bound for dropout that decays as $1 - \delta$.
Moreover, provided that $\delta$ is not too close to 1,
dropout behaves similarly to an adaptive $\LII$ regularizer
with parameter $\delta/(1 - \delta)$
\cite{wager2013dropout,baldi2014dropout}, and at least in
linear regression such $\LII$ regularization improves
generalization error by a constant factor.
In contrast, by leveraging the conditional independence assumptions
of the topic model, we are able to improve the \emph{exponent} in the rate of
convergence of the empirical risk minimizer.

It is also possible to analyze dropout as an adaptive regularizer
\cite{wager2013dropout,van2013learning,globerson2006nightmare}:
in comparison with $\LII$ regularization, dropout
favors the use of rare features and encourages confident predictions. If we believe that
good document classification should produce confident predictions by
understanding rare words with Poisson-like occurrence patterns, then the
work on dropout as adaptive regularization and our generalization-based analysis
are two complementary explanations for the success of dropout in
natural language tasks.

\sectshrink
\section{Dropout Training for Topic Models}
\label{sec:dropout}
\sectshrink

In this section, we introduce \emph{binomial dropout}, a form of dropout
suitable for topic models, and the Poisson topic model,
on which all our analyses will be based.

\sectshrink
\paragraph{Binomial Dropout}
\label{sec:binom}

Suppose that we have a binary classification problem\footnote{Dropout training is known
to work well in practice for multi-class problems \cite{wang2013feature}.
For simplicity, however, we will restrict our theoretical analysis to a two-class setup.}
 with count features $x^{(i)} \in \{0,1,2,\ldots\}^d$ and labels $y^{(i)} \in \{0, \, 1\}$.
For example, $x^{(i)}_j$ is the number of times the $j$-th word
in our dictionary appears in the $i$-th document,
and $y^{(i)}$ is the label of the document.
Our goal is to train a weight vector $\hw$ that
classifies new examples with features $x$
via a linear decision rule $\hy = \I\{\hw \cdot x > 0\}$.
We start with the usual empirical risk minimizer:
\begin{equation}
  \hw_0 \eqdef \argmin_{w \in \R^d}\left\{\sum_{i = 1}^n \ell\left(w; \, x^{(i)}, \, y^{(i)}\right)\right\}
\end{equation}
for some loss function $\ell$ (we will analyze the zero-one loss but use logistic loss in experiments, e.g., \cite{ng2001discriminative,bartlett2006convexity,zhang2004statistical}).
Binomial dropout trains on perturbed
features $\tx^{(i)}$ instead of the original features $x^{(i)}$:
\begin{align}
\label{eq:dropout}
\hw_{\delta} \eqdef \argmin_w\left\{\sum_{i = 1}^n \EE{\ell\left(w;\, {\tilde{x}^{(i)}}, \, y^{(i)}\right)}\right\},
\where {\tilde{x}_j^{(i)}} = \Binom\p{x_j^{(i)}; 1 - \delta}.
\end{align}
In other words, during training, we randomly thin the $j$-th feature $x_j$ with binomial noise.
If $x_j$ counts the number of times the $j$-th word appears in the document,
then replacing $x_j$ with $\tx_j$ is equivalent to
independently deleting each occurrence of word $j$ with probability $\delta$.
Because we are only interested in the decision boundary,
we do not scale down the weight vector obtained by dropout by a factor
$1 - \delta$ as is often done \cite{hinton2012improving}.

Binomial dropout differs slightly from the usual definition of (blankout) dropout,
which alters the feature vector $x$ by setting random coordinates to 0
\cite{wager2013dropout,van2013learning,mcallester2013pac,baldi2014dropout}.
The reason we chose to study binomial rather than blankout dropout is that Poisson
random variables remain Poisson even after binomial thinning; this fact lets us streamline our analysis.
For rare words that appear once in the document,
the two types of dropout are equivalent.

\sectshrink
\paragraph{A Generative Poisson Topic Model}
\label{sec:gen}

Throughout our analysis, we assume that the data is drawn from a
Poisson topic model depicted in Figure \ref{fig:graph} and defined as follows.
Each document $i$ is assigned a label $y^{(i)}$ according to some Bernoulli distribution.
Then, given the label $y^{(i)}$, the document gets a topic $\taui \in \Theta$
from a distribution $\rho_{y^{(i)}}$.
Given the topic $\tau^{(i)}$, for every word $j$ in the vocabulary,
we generate its frequency $x^{(i)}_j$
according to $x^{(i)}_j \cond \taui \sim \Pois(\lambda_j^{(\taui)})$,
where $\lambda_j^{(\tau)} \in [0, \infty)$ is the expected number of times
word $j$ appears under topic $\tau$.
Note that $\|\lambda^{(\tau)}\|_1$ is the average length of a document with topic $\tau$.
Define
$\lambda \eqdef \min_{\tau \in \Theta} \|\lambda^{(\tau)}\|_1$
to be the shortest average document length across topics.
If $\Theta$ contains only two topics---one for each class---we get the naive Bayes model.
If $\Theta$ is the $(K-1)$-dimensional simplex where $\lambda^{(\tau)}$ is a
$\tau$-mixture over $K$ basis vectors, we get
the $K$-topic latent Dirichlet allocation \cite{blei2003latent}.\footnote{
In topic modeling, the vertices of the simplex $\Theta$ are
``topics'' and $\tau$ is a mixture of topics, whereas we call $\tau$ itself a topic.}

Note that although our generalization result relies on a generative model,
the actual learning algorithm is agnostic to it.
Our analysis shows that
dropout can take advantage of a generative structure
while remaining a discriminative procedure.
If we believed that a certain topic model held exactly and we knew the
number of topics,  we could try to fit the full generative model by EM. This,
however, could make us vulnerable to model misspecification.
In contrast, dropout benefits from
generative assumptions while remaining more robust to misspecification.

\sectshrink
\section{Altitude Training: Linking the Dropout and Data-Generating Measures}
\label{sec:link}
\sectshrink

Our goal is to understand the behavior of a classifier $\hath_\delta$ trained using dropout.
During dropout, the error of any classifier $h$ is characterized by two measures.
In the end, we are interested in the usual generalization error (expected risk) of $h$
where $x$ is drawn from the underlying \emph{data-generating measure}:
\begin{equation}
\label{eq:err}
\Err\p{h} \eqdef \PP{y \neq h(x)}.
\end{equation}
However, since dropout training works on the corrupted data $\tx$ (see \eqref{eq:dropout}),
in the limit of infinite data,
the dropout estimator will converge to the minimizer of the generalization error
with respect to the \emph{dropout measure} over $\tx$:
\begin{equation}
\label{eq:err_delta}
\Err_\delta\p{h} \eqdef \PP{y \neq h(\tx)}.
\end{equation}

The main difficulty in analyzing the generalization of dropout is that
classical theory tells us that the generalization error with respect to
the dropout measure will decrease as $n \rightarrow \infty$,
but we are interested in the original measure.
Thus, we need to bound $\Err$ in terms of $\Err_\delta$.
In this section, we show that the error on the original measure is actually much smaller
than the error on the dropout measure; we call this the \emph{altitude
training phenomenon}.

Under our generative model, the count features $x_j$ are conditionally independent
given the topic $\tau$.
We thus focus on a single fixed topic $\tau$
and establish the following theorem, which provides a per-topic analogue of
\eqref{eq:drop_gen}.
Section~\ref{sec:genbound} will then use this theorem to obtain our main result.

\begin{theo}
\label{theo:err}
Let $h$ be a binary linear classifier with weights $w$, and suppose that our features
are drawn from the Poisson generative model given topic $\tau$. Let $c_\tau$ be
the more likely label given $\tau$:
\begin{equation}
\label{eq:ltau}
c_\tau \eqdef \arg\max_{c \in \{0,1\}} \PP{y^{(i)} = c \cond \tau^{(i)} = \tau}.
\end{equation}
Let $\tvarepsilon_\tau$ be the
sub-optimal prediction rate in the dropout measure
\begin{equation}
\label{eq:baseline-err}
\tvarepsilon_\tau \eqdef \PP{\I\left\{w \cdot \tx^{(i)} > 0\right\} \neq c_\tau \cond \taui = \tau},
\end{equation}
where $\tx^{(i)}$ is an example thinned by binomial dropout \eqref{eq:dropout},
and $\mathbb{P}$ is taken over the data-generating process.
Let $\varepsilon_\tau$ be the sub-optimal prediction rate in the original measure
\begin{align}
\label{eq:dropout-err}
\varepsilon_\tau \eqdef \PP{\I\left\{w \cdot x^{(i)} > 0\right\}\neq c_\tau \cond \taui = \tau}.
\end{align}
Then:
\begin{align}
\label{eq:error-cut}
\varepsilon_\tau = \too \p{\tvarepsilon_\tau^{\frac{1}{1 - \delta}} + \sqrt{\Psi_\tau}},
\end{align}
where
$\Psi_\tau = {\max_j\left\{w_j^2\right\}} / {\sum_{j = 1}^d \lambda^{(\tau)}_j w_j^2}$,
and the constants in the bound depend only on $\delta$.
\end{theo}

Theorem \ref{theo:err} only provides us with a useful bound when the term
$\Psi_\tau$ is small. 
Whenever the largest $w^2_j$ is not much larger than the average $w^2_j$,
then $\sqrt{\Psi_\tau}$ scales as $O(1/\sqrt{\lambda})$,
where $\lambda$ is the average document length.
Thus, the bound \eqref{eq:error-cut}
is most useful for long documents.

\begin{figure}[t]
\figshrink
\begin{subfigure}[b]{0.5\columnwidth}
\begin{center}
\begin{tikzpicture}
\clip (-1.3,-1.5) rectangle (6.25, 3);
\tikzstyle{main}=[circle, minimum size = 9mm, thick, draw =black!80, node distance = 12.5mm]
\tikzstyle{connect}=[-latex, thick]
\tikzstyle{box}=[rectangle, draw=black!100]
  \node[main, fill=black!10] (theta) [label=below:$y$] { };
  \node[main] (z) [right=of theta,label=below:$\tau$] {};
  \node[main, fill = black!10] (w) [right=of z,label=below:$x$] { };
  \node[main] (beta) [above=of z,label=below:$\lambda$] { };
  \node[main] (pi) [above=of theta,label=below:$\rho$] { };
  \path
        (pi) edge [connect] (z)
        (theta) edge [connect] (z)
        (z) edge [connect] (w)
        (beta) edge [connect] (w);
  \node[rectangle, inner sep=0mm, fit= (z) (w),label=below right:$J$, xshift=11mm] {};
  \node[rectangle, inner sep=4.4mm,draw=black!100, fit= (w)] {};
  \node[rectangle, inner sep=3.6mm, fit= (z) (w) ,label=below right:$I$, xshift=11.5mm] {};
  \node[rectangle, inner sep=7.7mm, draw=black!100, fit = (theta) (z) (w)] {};
\end{tikzpicture}
\caption{Graphical representation of the Poisson topic model:
Given a document with label $y$, we draw a document topic $\tau$ from the multinomial distribution with probabilities $\rho_y$.
Then, we draw the words $x$ from the topic's Poisson distribution with mean $\lambda^{(\tau)}$.
Boxes indicate repeated observations, and greyed-out nodes are observed during training.}
\label{fig:graph}
\end{center}
\end{subfigure}
\hspace{0.02\columnwidth}
\begin{subfigure}[b]{0.45\columnwidth}
\centering
\includegraphics[width=0.8\columnwidth]{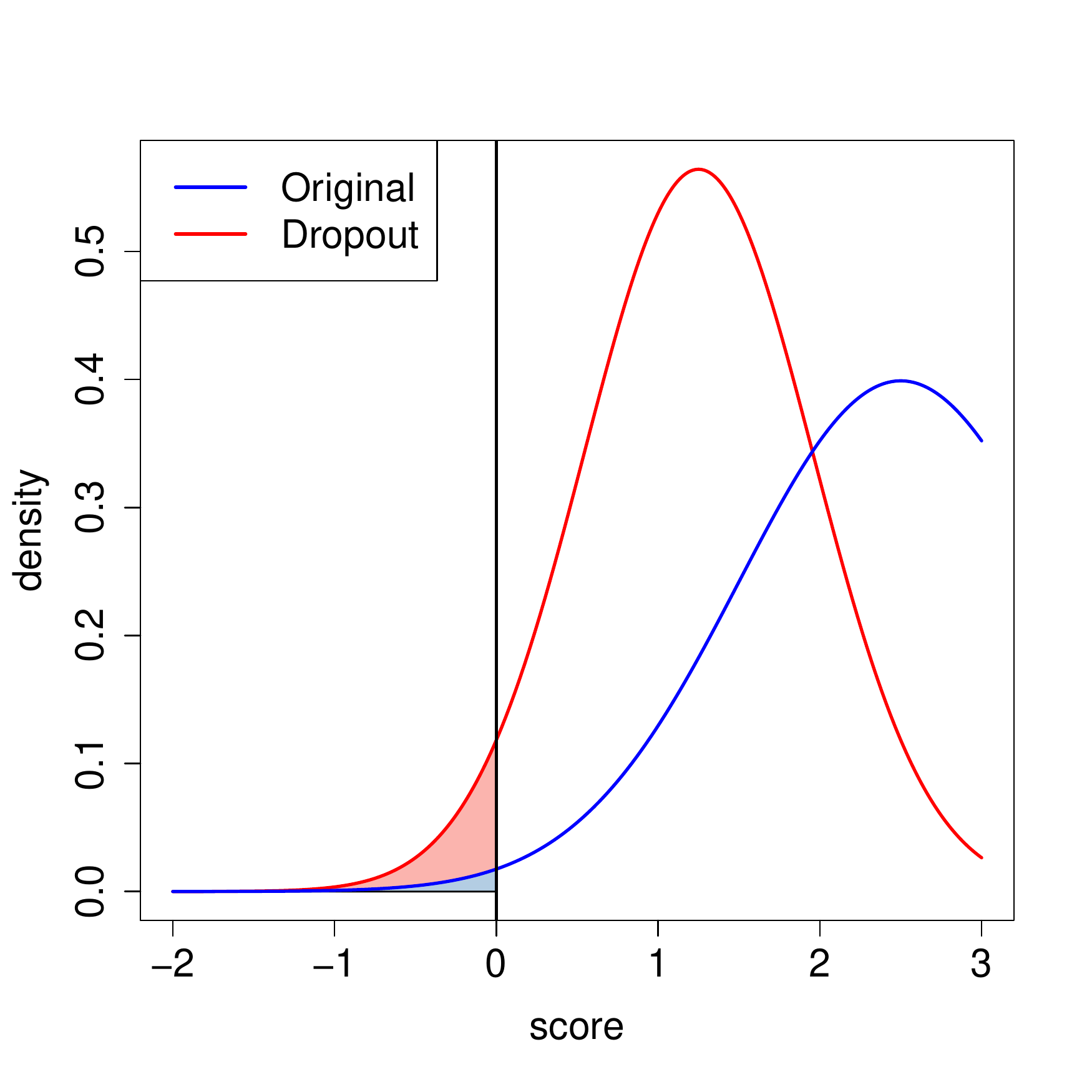}
\caption{
For a fixed classifier $w$,
the probabilities of error on an example drawn from the original and dropout measures
are governed by the tails of two Gaussians (shaded).
The dropout Gaussian has a larger coefficient of variation,
which means the error on the original measure (test) is much smaller
than the error on the dropout measure (train) \eqref{eq:approx2}.
In this example, $\mu = 2.5$, $\sigma = 1$, and $\delta = 0.5$.}
\label{fig:gauss_cmp}
\end{subfigure}
\caption{(a) Graphical model.  (b) The altitude training phenomenon.}
\label{fig:topic}
\figshrink
\end{figure}
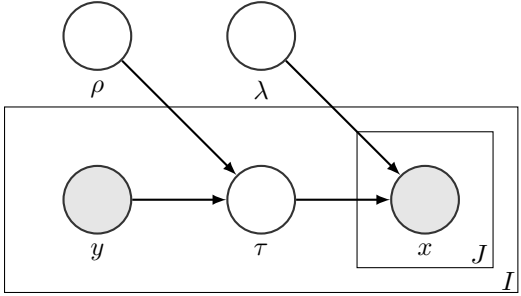
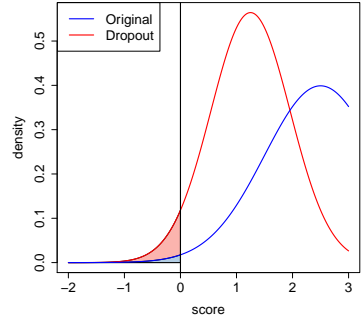

\sectshrink
\paragraph{A Heuristic Proof of Theorem \ref{theo:err}.}

The proof of Theorem \ref{theo:err} is provided in the technical appendix.
Here, we provide a heuristic argument for intuition.
Given a fixed topic $\tau$, suppose that it is optimal to predict $c_\tau = 1$, so our
test error is
$\varepsilon_\tau = \PP{w \cdot x \leq 0 \cond \tau}.$
For long enough documents, by the central limit theorem,
the score $s \eqdef w \cdot x$ will be roughly Gaussian
$s \sim \nn\p{\mu_\tau, \, \sigma_\tau^2},$
where
$\mu_\tau = \sum_{j = 1}^d \lambda^{(\tau)}_j w_j$
and
$\sigma^2_\tau = \sum_{j = 1}^d \lambda^{(\tau)}_j w_j^2.$
This implies that
$ \varepsilon_\tau \approx \Phi\p{-{\mu_\tau}/{\sigma_\tau}},$
where $\Phi$ is the cumulative distribution function of the Gaussian.
Now, let $\ts \eqdef w \cdot \tx$ be the score on a dropout sample. Clearly,
$\EE{\ts} = (1 - \delta)\,\mu_\tau$ and $\Var{\ts} = (1 - \delta)\,\sigma^2_\tau$,
because the variance of a Poisson random variable scales with its mean.
Thus,
\begin{equation}
\label{eq:approx2}
\tvarepsilon_\tau
\approx \Phi\p{- {\sqrt{1 - \delta}}\, \frac{\mu_\tau}{\sigma_\tau}}
\approx \Phi\p{-\frac{\mu_\tau}{\sigma_\tau}}^{(1 - \delta)}
\approx \varepsilon_\tau^{(1 - \delta)}.
\end{equation}
Figure~\ref{fig:gauss_cmp} illustrates the relationship between the two Gaussians.
This explains the first term on the right-hand side
of \eqref{eq:error-cut}.
The extra error term $\sqrt{\Psi_\tau}$ arises from a Berry-Esseen bound
that approximates Poisson mixtures by Gaussian random variables.

\sectshrink
\section{A Generalization Bound for Dropout}
\label{sec:genbound}
\sectshrink

By setting up a bridge between the dropout measure and the original data-generating measure,
Theorem \ref{theo:err} provides a foundation for our analysis. It remains to
translate this result into a statement about the generalization error of dropout.
For this, we need to make a few assumptions.

Our first assumption is fundamental:
if the classification signal is concentrated among just a few features, then we
cannot expect dropout training to do well. The second and third assumptions,
which are more technical, guarantee that a classifier can only do well overall
if it does well on every topic; this lets us apply Theorem \ref{theo:err}. A more
general analysis that relaxes Assumptions 2 and 3 may be an interesting avenue
for future work.

\sectshrink
\paragraph{Assumption 1: well-balanced weights}

First, we need to assume that all the signal is not concentrated in a few features.
To make this intuition
formal, we say a linear classifier with weights $w$ is \emph{well-balanced}
if the following holds for each topic $\tau$:
\begin{equation}
\label{eq:well-conditioned}
\frac{{\max_j\left\{w_j^2\right\}}\,\sum_{j = 1}^d \lambda^{(\tau)}_j }{\sum_{j = 1}^d \lambda^{(\tau)}_j w_j^2} \leq \kappa \text{ for some } 0 < \kappa < \infty.
\end{equation}
For example, suppose each word was either useful ($|w_j| = 1$) or not ($w_j = 0$);
then $\kappa$ is the inverse expected fraction of words in a document that are useful.
In Theorem \ref{theo:main} we restrict the ERM to
well-balanced classifiers and assume that the expected risk minimizer
$h^*$ over all linear rules is also well-balanced.

\sectshrink
\paragraph{Assumption 2: discrete topics}
Second, we assume that there are a finite number $T$ of
topics, and that the available topics are not too rare or ambiguous:
the minimal probability of observing any topic $\tau$ is bounded below by
\begin{equation}
\label{eq:mintau}
\PP{\tau} \geq p_\text{\rm min} > 0,
\end{equation}
and that each topic-conditional probability is bounded away from $\frac12$ (random guessing):
\begin{equation}
\label{eq:minconf}
\left\lvert \PP{y^{(i)} = c \cond \taui = \tau} - \frac12 \right\rvert \geq \alpha > 0
\end{equation}
for all topics $\tau \in \{1, \, ..., \, T\}$. This assumption substantially
simplifies our arguments, allowing us to apply Theorem \ref{theo:err} to each topic
separately without technical overhead.

\sectshrink
\paragraph{Assumption 3: distinct topics}

Finally, as an extension of Assumption 2, we require that the topics be ``well separated.''
First, define $\Err_\text{min} = \mathbb{P}[y^{(i)} \neq
c_{\taui}]$,
where $c_\tau$ is the most likely label given topic $\tau$ \eqref{eq:ltau};
this is the error rate of the optimal decision rule that sees topic $\tau$.
We assume that the best linear rule $h^*_\delta$ satisfying \eqref{eq:well-conditioned}
is almost as good as always guessing the best label $c_\tau$ under the dropout measure:
\begin{equation}
\label{eq:shatter}
\Err_\delta\p{h^*_\delta} = \Err_\text{min} +\, \oo\p{\frac{1}{\sqrt{\lambda}}},
\end{equation}
where, as usual, $\lambda$ is a lower bound on the average document length.
If the dimension $d$ is larger than the number of topics $T$, this assumption is fairly weak:
the condition \eqref{eq:shatter} holds whenever the matrix $\Pi$ of topic
centers has full rank, and the minimum singular value of $\Pi$ is not too
small (see Proposition \ref{prop:shatter} in the Appendix for details).
This assumption is satisfied if the different topics can be separated from each other with a large margin.

Under Assumptions~1--3 we can turn Theorem~\ref{theo:err} into a statement
about generalization error.

\begin{theo}
\label{theo:main}
Suppose that our features $x$ are drawn
from the Poisson generative model (Figure \ref{fig:graph}),
and Assumptions 1--3 hold.
Define the excess risks of the dropout classifier $\hat h_\delta$ on the dropout and
data-generating measures, respectively:
\begin{equation}
\teta \eqdef \Err_\delta\p{\hath_\delta} - \Err_\delta\p{h^*_\delta}
\ \eqand \
\eta \eqdef \Err\p{\hath_\delta} - \Err\p{h^*_\delta}.
\end{equation}
Then, the altitude training phenomenon applies:
\begin{align}
\label{eq:droperr}
\eta = \too \p{\teta^{\frac{1}{1 - \delta}} + \frac{1}{\sqrt{\lambda}}}.
\end{align}
The above bound scales linearly in $p_\text{min}^{-1}$ and $\alpha^{-1}$;
the full dependence on $\delta$ is shown in the appendix. 
\end{theo}

In a sense, Theorem~\ref{theo:main} is a meta-generalization bound
that allows us to transform generalization bounds with respect
to the dropout measure ($\tilde\eta$) into ones on the data-generating measure ($\eta$)
in a modular way.
As a simple example,
standard VC theory provides an $\tilde\eta = \too_P(\sqrt{d/n})$ bound which,
together with Theorem~\ref{theo:main}, yields:

\begin{coro}
\label{coro:main}
Under the same conditions as Theorem \ref{theo:main},
the dropout classifier $\hath_\delta$ achieves the following excess risk:
\vspace{-3mm}
\begin{align}
\label{eq:coro:main}
\Err\p{\hath_\delta} - \Err\p{h_\delta^*}  = \too_P\p{\p{\sqrt{{\frac{d}{n}}}}^{\frac{1}{1 - \delta}} + \frac{1}{\sqrt{\lambda}} }.
\end{align}
\end{coro}

More generally, we can often check that upper bounds for
$\Err(\hath) - \Err(h^*)$ also work as upper bounds for
$ \Err_\delta(\hath_\delta) - \Err_\delta(h^*_\delta)$; this gives us
the heuristic result from \eqref{eq:drop_gen}.

\sectshrink
\section{The Bias of Dropout}
\label{sec:geom}
\sectshrink

In the previous section, we showed that under the Poisson topic
model in Figure \ref{fig:graph}, dropout can achieve a
substantial cut in excess risk $\Err(\hath_\delta) - \Err(h_\delta^*)$.
But to complete our picture
of dropout's performance, we must address the bias of dropout:
$\Err(h_\delta^*) - \Err(h^*)$.

Dropout can be viewed as importing ``hints'' from a
generative assumption about the data. Each observed $(x,y)$ pair (each
labeled document) gives us information not only about the conditional class
probability at $x$, but also about the conditional class probabilities
at numerous other hypothetical values $\tx$ representing
shorter documents of the same class that did not occur.
Intuitively, if these $\tx$ are actually good representatives of that class,
the bias of dropout should be mild.

For our key result in this section, we will take the Poisson generative model
from Figure \ref{fig:graph}, but further assume that document length is
independent of the topic.
Under this assumption, we will show that dropout preserves the Bayes
decision boundary in the following sense:
\begin{prop}
\label{prop:cond}
Let $(x,y)$ be distributed according to the Poisson topic model of Figure~\ref{fig:graph}.
Assume that document length is independent of topic:
$\|\lambda^{(\tau)}\|_1 = \lambda$ for all topics $\tau$.
Let $\tx$ be a binomial dropout sample of $x$ with some dropout probability $\delta \in (0, 1)$.
Then, for every feature vector $v \in \R^d$, we have:
  \begin{align}
     \label{eq:balance_cond}
    \PP{y=1 \cond \tx=v} = \PP{y=1 \cond x=v}.
  \end{align}
\end{prop}

If we had an infinite amount of data $(\tx,y)$ corrupted under dropout,
we would predict according to $\I\{ \PP{y = 1 \cond \tx=v} > \frac12 \}$.
The significance of Proposition~\ref{prop:cond} is that this decision
rule is identical to the true Bayes decision boundary (without dropout).
Therefore, the empirical risk minimizer of a sufficiently
rich hypothesis class trained with dropout would incur very small bias.

However, Proposition~\ref{prop:cond} does {\em not} guarantee
that dropout incurs no bias when we fit a linear classifier.
In general, the best linear
approximation for classifying shorter documents is not necessarily the best for
classifying longer documents.  As $n\to\infty$, a linear classifier
trained on $(x,y)$ pairs will eventually
outperform one trained on $(\tx, y)$ pairs.

\begin{figure}[t]
\figshrink
\begin{subfigure}[b]{0.45\columnwidth}
\centering
  \includegraphics[width=\columnwidth]{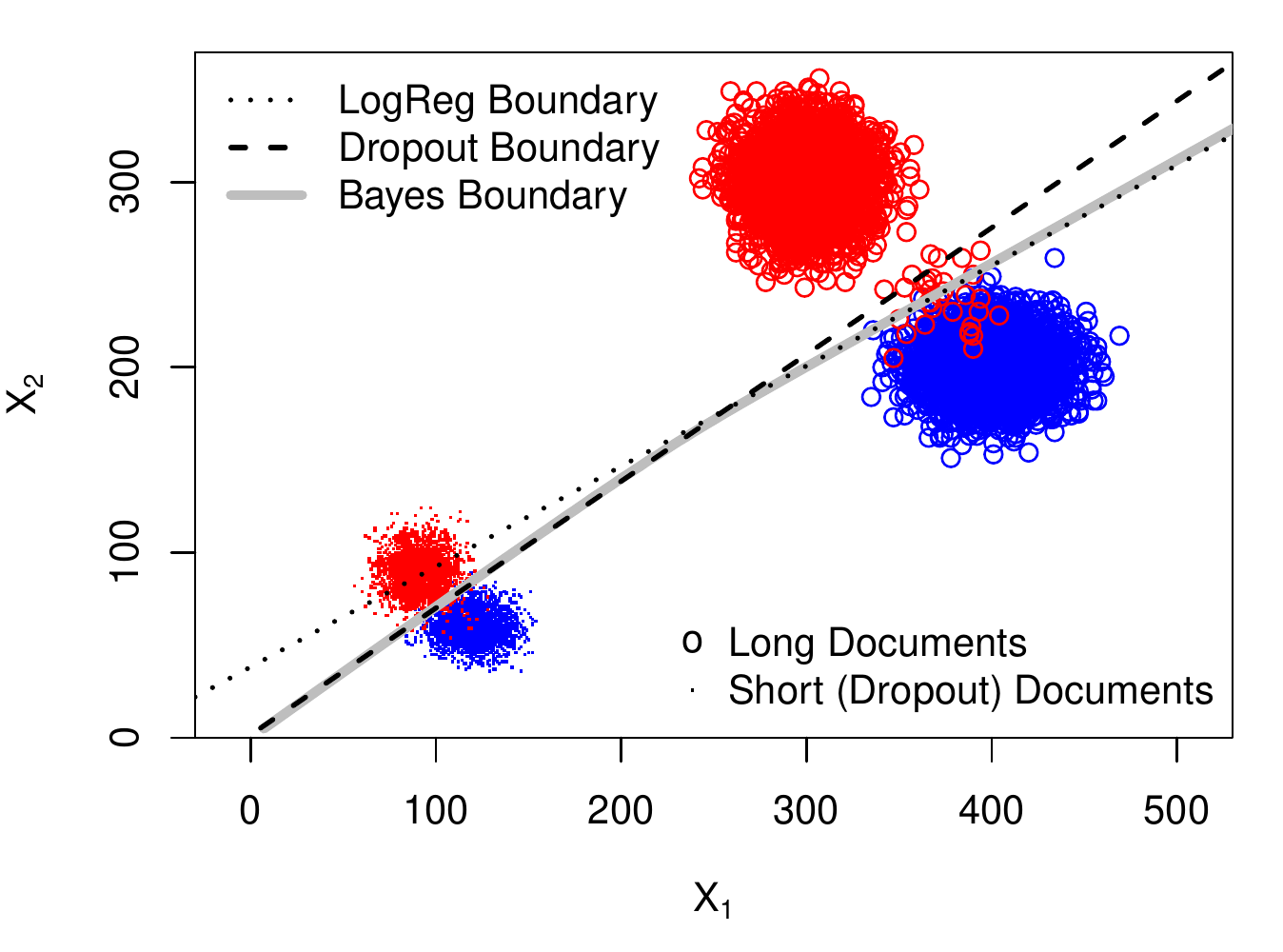}
  \vspace{-6mm}
  \caption{Dropout ($\delta=0.75$) with $d=2$.  For
    long documents (circles in the upper-right),
    logistic regression focuses on capturing the
    small red cluster; the large red cluster has almost no
    influence.  Dropout (dots in the lower-left) distributes
    influence more equally between the two red clusters.}
  \label{fig:dropInfl}
\end{subfigure}
\hspace{0.05\columnwidth}
\begin{subfigure}[b]{0.46\columnwidth}
  \centering
  \includegraphics[width=1\columnwidth]{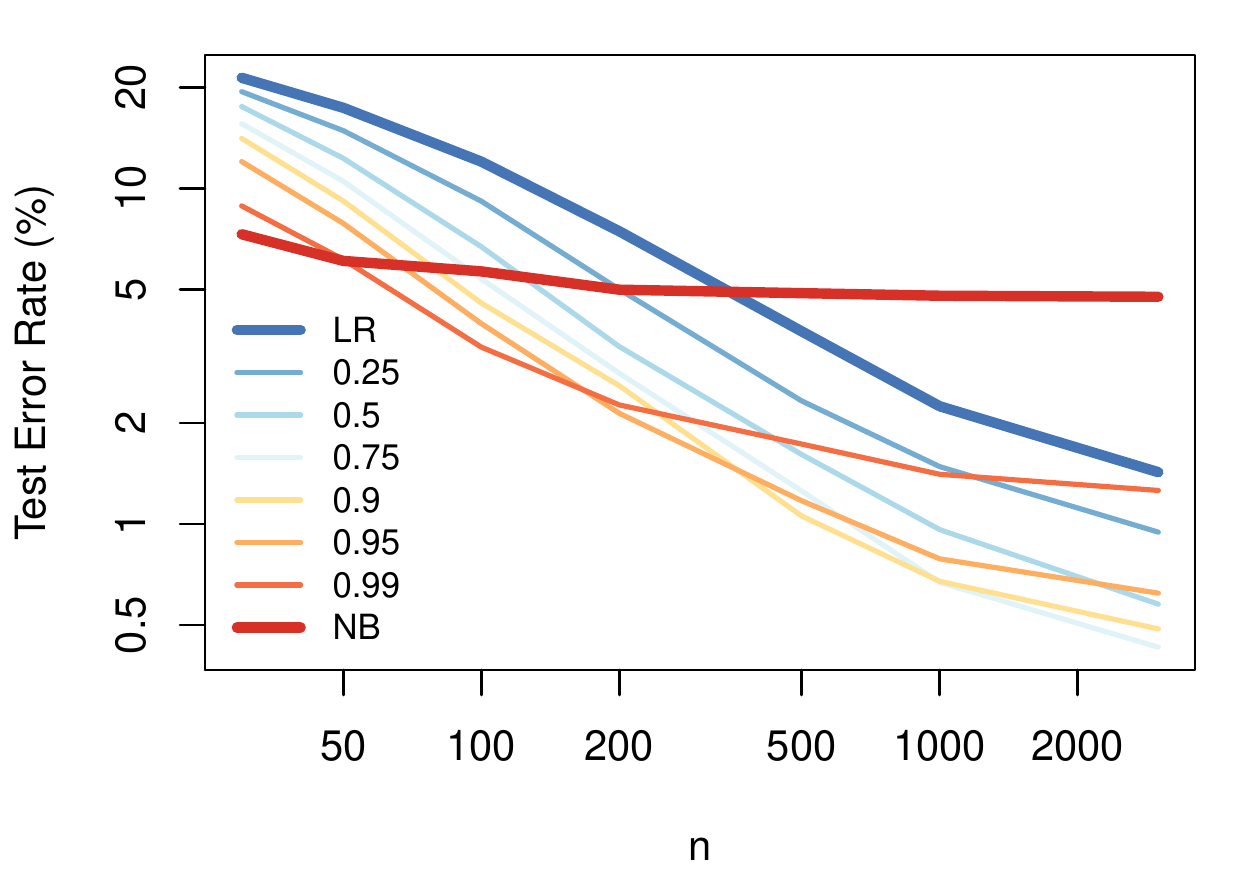}
  \caption{Learning curves for the synthetic experiment.
    Each axis is plotted on a log scale.  Here the dropout rate $\delta$ ranges
    from 0 (logistic regression) to 1 (naive Bayes) for multiple
    values of training set sizes $n$.
    As $n$ increases, less dropout is preferable,
    as the bias-variance tradeoff shifts.
    }
  \label{fig:multiMixSim}
\end{subfigure}
\caption{Behavior of binomial dropout in simulations. In the left panel, the circles are the original data, while the dots are dropout-thinned examples.  The Monte Carlo error is negligible.}
\label{fig:simu}
\figshrink
\end{figure}

\paragraph{Dropout for Logistic Regression}
To gain some more intuition about how dropout affects linear
classifiers, we consider logistic regression. A similar
phenomenon should also hold for the ERM, but discussing this solution
is more difficult since the ERM solution does not have have a simple
characterization.
The relationship between the 0-1 loss and convex surrogates has been
studied by, e.g., \cite{bartlett2006convexity,zhang2004statistical}.
The score criterion for logistic regression is
$0 = \sum_{i=1}^n \left(y^{(i)} - \hat p_i\right)x^{(i)}$,
where $\hat p_i = (1+e^{-\hw \cdot x^{(i)}})^{-1}$ are the fitted probabilities.
Note that easily-classified examples (where $\hat p_i$ is close to $y^{(i)}$) play almost
no role in driving the fit.  Dropout turns
easy examples into hard examples, giving more examples a chance to
participate in learning a good classification rule.

Figure \ref{fig:dropInfl} illustrates dropout's tendency to spread
influence more democratically
for a simple classification problem with $d=2$.
The red class is a 99:1 mixture over two topics, one of
which is much less common, but harder to classify, than the other.
There is only one topic for the blue class.
For long documents (open circles in the top right), the infrequent,
hard-to-classify
red cluster dominates the fit while the frequent, easy-to-classify red
cluster is essentially ignored.  For dropout documents with
$\delta=0.75$ (small dots, lower left),
both red clusters are relatively hard to classify, so the infrequent
one plays a less disproportionate role in driving the fit.  As a
result, the fit based on dropout is more stable but
misses the finer structure near the decision boundary.
Note that the solid gray curve, the Bayes boundary, is unaffected
by dropout, per Proposition \ref{prop:cond}.  But,
because it is nonlinear, we obtain a different linear approximation
under dropout.

\sectshrink
\section{Experiments and Discussion}
\sectshrink


\paragraph{Synthetic Experiment}
Consider the following instance of the Poisson topic model:
We choose the document label uniformly at random: $\PP{y^{(i)} = 1} = \frac12$.
Given label $0$, we choose topic $\taui = 0$ deterministically;
given label $1$, we choose a real-valued topic $\taui \sim \Exp(3)$.
The per-topic Poisson intensities $\lambda^{(\tau)}$ are defined as follows:
\begin{align}
\theta^{(\tau)} = \begin{cases}
(1, \, \ldots, \,1 \,\, \cond \,\,
 0,\, \ldots,\, 0 \,\, \cond \,\,
 0,\, \ldots,\, 0) & \text{ if } \tau = 0, \\
(\underbrace{0, \, \ldots, \,0}_{7} \, \cond \,
 \underbrace{\tau,\, \ldots,\, \tau}_7 \,\cond\,
 \underbrace{0,\, \ldots,\, 0}_{486}) & \text{ otherwise,}
\end{cases}
\quad
\lambda^{(\tau)}_j = 1000 \cdot \frac{e^{\theta^{(\tau)}_j}}{\sum_{j'=1}^{500} e^{\theta^{(\tau)}_{j'}}}.
\end{align}
The first block of 7 independent words are indicative of label 0,
the second block of 7 \emph{correlated} words are indicative of label 1,
and the remaining 486 words are indicative of neither.

We train a model on training sets of various size $n$, and evaluate
the resulting classifiers' error rates on a large test set.
For dropout, we recalibrate the intercept on the training set.
Figure \ref{fig:multiMixSim} shows the results.
There is a clear bias-variance tradeoff, with logistic regression
(${\delta=0}$) and naive Bayes (${\delta=1}$) on the two ends of the
spectrum.\footnote{When the logistic regression fit is degenerate, we use $\LII$-regularized logistic regression with weight $10^{-7}$.}
For moderate values of $n$, dropout improves performance, with ${\delta=0.95}$ (resulting
in roughly 50-word documents) appearing nearly optimal for this example.

\begin{figure}[t]
\figshrink
\begin{subfigure}[b]{0.45\columnwidth}
\centering
\includegraphics[width = \columnwidth]{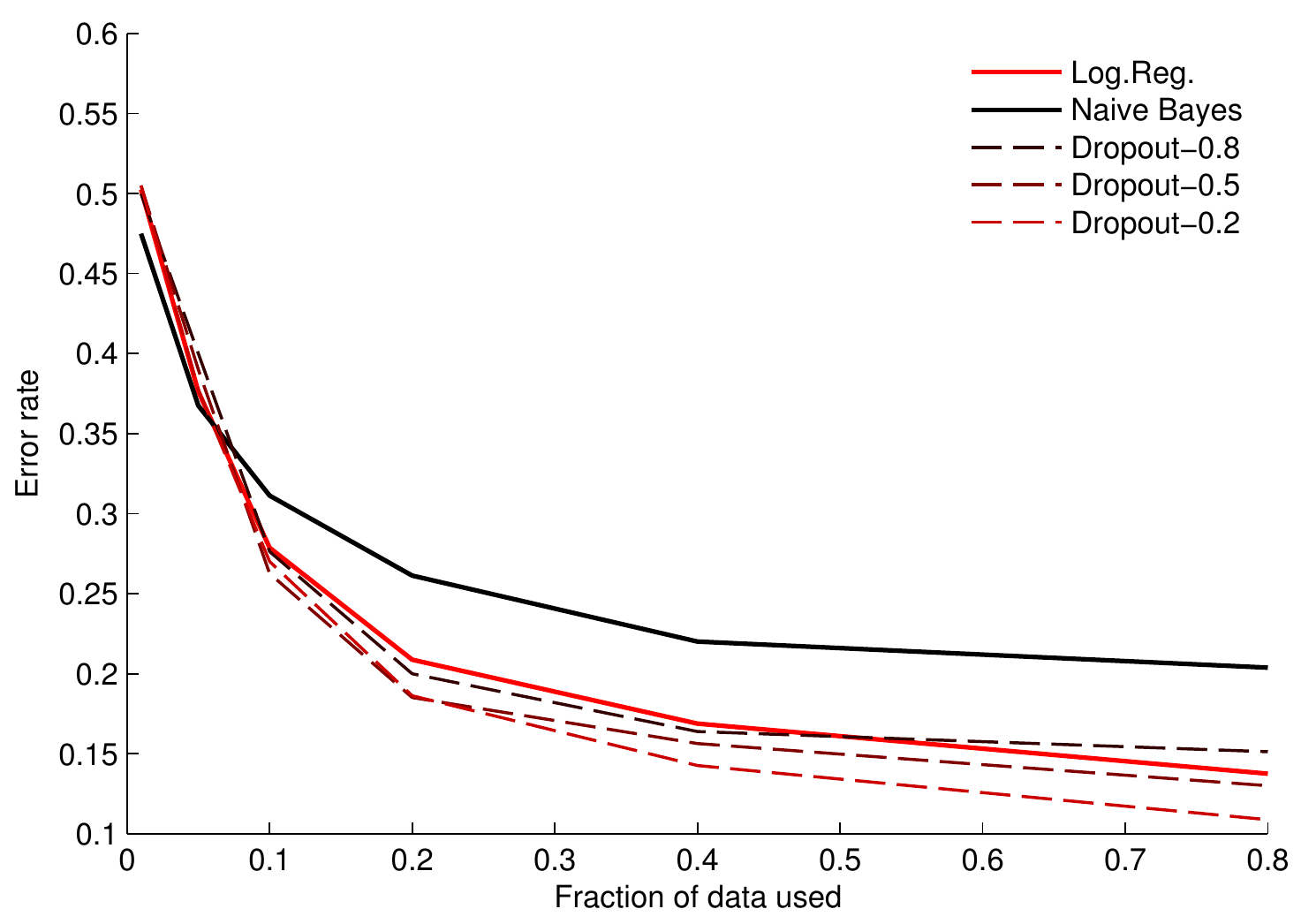}
\caption{Polarity 2.0 dataset \cite{pang2004sentimental}.}
\label{fig:polarity}
\end{subfigure}
\hspace{0.05\columnwidth}
\begin{subfigure}[b]{0.45\columnwidth}
\centering
\includegraphics[width = \columnwidth]{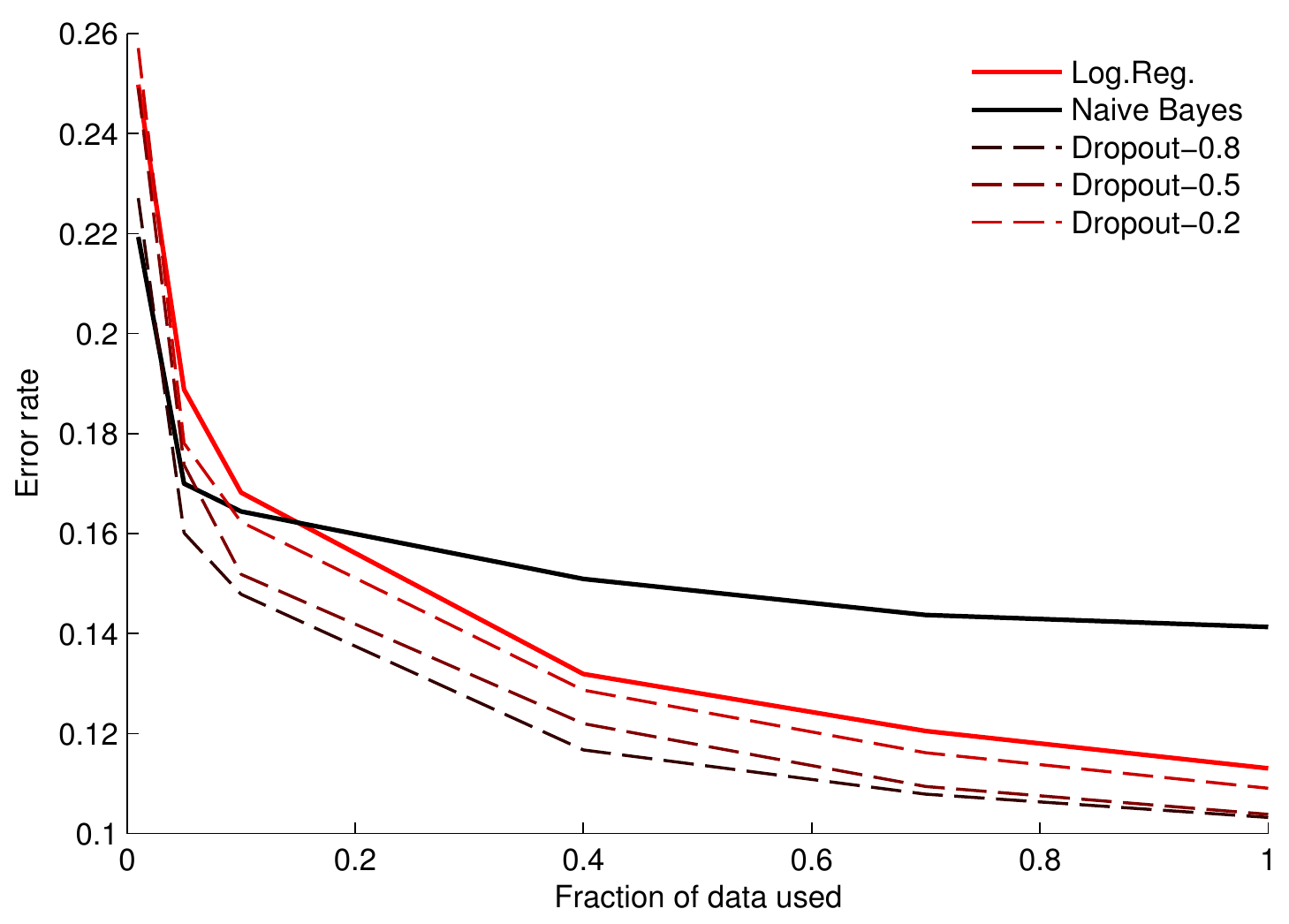}
\caption{IMDB dataset \cite{maas2011learning}.}
\label{fig:imdb}
\end{subfigure}
\caption{Experiments on sentiment classification.
More dropout is better relative to logistic regression
for small datasets and gradually worsens with more training data.
}
\label{fig:nlp}
\figshrink
\end{figure}

\sectshrink
\paragraph{Sentiment Classification}

We also examined the performance of dropout as a function of training set size
on a document classification task. Figure \ref{fig:polarity} shows results on
the Polarity 2.0 task \cite{pang2004sentimental}, where the goal is to
classify positive versus negative movie reviews on IMDB. We divided the dataset
into a training set of size 1,200 and a test set of size 800, and
trained a bag-of-words
logistic regression model with 50,922 features. This example exhibits
the same behavior as our
simulation. Using a larger $\delta$ results in a classifier that
converges faster at
first, but then plateaus.
We also ran experiments on
a larger IMDB dataset \cite{maas2011learning}
with training and test sets of size 25,000 each and approximately 300,000 features.
As Figure \ref{fig:imdb} shows,
the results are similar, although the training set is not large enough
for the learning curves to cross. When using the full training set, all but three pairwise comparisons in Figure~\ref{fig:nlp} are statistically significant ($p<0.05$ for McNemar's test).

\sectshrink
\paragraph{Dropout and Generative Modeling}

Naive Bayes and empirical risk minimization represent two divergent approaches to
the classification problem. ERM is guaranteed to find the
best model as $n \rightarrow \infty$ but can have suboptimal generalization
error when $n$ is not large relative to $d$. Conversely, naive Bayes has very
low generalization error, but suffers from asymptotic bias. In this
paper, we showed that dropout behaves as a link between ERM
and naive Bayes, and can sometimes achieve a
more favorable bias-variance tradeoff.
By training on randomly generated sub-documents
rather than on whole documents, dropout implicitly codifies a
generative assumption about the data, namely that excerpts from a
long document should have the same label as the original document
(Proposition~\ref{prop:cond}).

Logistic regression with dropout appears to have
an intriguing connection to the naive Bayes SVM (NBSVM)
\cite{wang2012baselines}, which is a way of using naive
Bayes generative assumptions to strengthen an SVM. In a recent survey of bag-of-words
classifiers for document classification, NBSVM and dropout often obtain
state-of-the-art accuracies, e.g., \cite{wang2013fast}. This suggests that a good way to
learn linear models for document classification is to use discriminative
models that borrow strength from an approximate generative
assumption to cut their generalization error.
Our analysis presents an interesting contrast to other work that
directly combine generative and discriminative modeling by optimizing a
hybrid likelihood \cite{raina04hybrid,bouchard04tradeoff,lasserre06hybrid,bouchard2007bias,pal06mcl,liang2008asymptotic}.
Our approach is more guarded in that we only let the generative assumption
speak through pseudo-examples.

\sectshrink
\paragraph{Conclusion}

We have presented a theoretical analysis that explains how
dropout training can be very helpful under a Poisson topic model assumption.
Specifically, by making training examples artificially difficult, dropout improves
the exponent in the generalization bound for ERM.
We believe that this work is just the first step in understanding
the benefits of training with artificially corrupted features,
and we hope the tools we have developed can be extended to analyze
other training schemes under weaker data-generating assumptions.

\clearpage

{\small
\bibliographystyle{unsrt}
\bibliography{./references}
}

\clearpage

\begin{appendix}

\section{Technical Results}
\label{sec:technical}

We now give detailed proofs of the theorems in the paper.

\subsection{Altitude Training Phenomeon}

We begin with a proof of our main generalization bound result, namely Theorem \ref{theo:err}.
The proof is built on top of the following Berry-Esseen type result.

\begin{lemm}
\label{lemm:berry-esseen}
Let $Z_1$, ..., $Z_d$ be independent Poisson random variables with means $\lambda_j \in \RR_+$, and let
$$ S = \sum_{j = 1}^d w_jZ_j, \, \mu = \EE{S}, \eqand \sigma^2 = \Var{S} $$
for some fixed set of weights $\{w_j\}_{j=1}^d$. Then, writing $F_S$ for the distribution function of $S$ and $\Phi$ for the standard Gaussian distribution,
\begin{equation}
\label{eq:berry-esseen}
\sup_{x \in \RR} \left\lvert F_S(x) - \Phi\p{\frac{x - \mu}{\sigma}} \right\lvert \leq \CBE \sqrt{\frac{\max_j \{w_j^2\}}{\sum_{j = 1}^d \lambda_j w_j^2}},
\end{equation}
where $\CBE \le 4$.

\proof
Our first step is to write $S$ as a sum of bounded \emph{i.i.d.} random
variables. Let $N = \sum_{j = 1}^d Z_j$. Conditional on $N$, the $Z_j$ are
distributed as a multinomial with parameters $\pi_j = \lambda_j/\lambda$ where
$\lambda = \sum_{j=1}^d \lambda_j$. Thus,
$$ \law\p{S \cond N} \eqd \law\p{\sum_{k = 1}^N W_k \cond N}, $$
where $W_k \in \{w_1, \, ..., \, w_d\}$ is a single multinomial draw from the available weights with probability parameters $\PP{W_k = w_j} =\pi_j$. This implies that,
$$ S \eqd \sum_{k = 1}^N W_k, $$
where $N$ itself is a Poisson random variable with mean $\lambda$.

We also know that a Poisson random variable can be written as a limiting mixture of many rare Bernoulli trials:
$$ B^{(m)} \Rightarrow N, \with B^{(m)} = \Binom\p{m, \, \frac{\lambda}{m}}. $$
The upshot is that
\begin{equation}
\label{eq:poisson_bdd_sum}
S^{(m)} \Rightarrow S, \with S^{(m)} = \sum_{k = 1}^m W_k I_k,
\end{equation}
where the $W_k$ are as before, and the $I_k$ are independent Bernoulli draws with parameter $\lambda/m$. Because $S^{(m)}$ converges to $S$ in distribution, it suffices to show that \eqref{eq:berry-esseen} holds for large enough $m$. The moments of $S^{(m)}$ are correct in finite samples: $\EE{S^{(m)}} = \mu$ and $\Var{S^{(m)}} = \sigma^2$ for all $m$.

The key ingredient in establishing \eqref{eq:berry-esseen} is the Berry-Esseen inequality \cite{feller1971introduction}, which in our case implies that
$$ \sup_{x \in \RR} \left\lvert F_{S^{(m)}}(x) - \Phi\p{\frac{x - \mu}{\sigma}} \right\rvert \le \frac{\rho_m}{2 s_m^3 \sqrt{m}}, $$
where
\begin{align*}
&s_m^2 = \Var{W_kI_k}, \\
&\rho_m = \EE{\left\lvert W_kI_k - \EE{W_kI_k}\right\rvert^3},
\end{align*}
We can show that
\begin{align*}
&s_m^2 = \EE{\p{W_kI_k}^2} - \EE{W_kI_k}^2= \frac{\lambda}{m}\EE{W_k^2} - \p{\frac{\lambda}{m}\EE{W_k}}^2, \eqand \\
&\rho_m  \leq 8 \p{\EE{\left\lvert W_kI_k \right\rvert^3} + \EE{\left\lvert W_kI_k \right\rvert}^3}= 8\p{\frac{\lambda}{m} \EE{\left\lvert W_k \right\rvert^3} + \p{\frac{\lambda}{m} \EE{\left\lvert W_k \right\rvert}}^3 }.
\end{align*}
Taking $m$ to $\infty$, this implies that
$$ \sup_{x \in \RR} \left\lvert F_{S}(x) - \Phi\p{\frac{x - \mu}{\sigma}} \right\rvert \le \frac{4 \EE{\left\lvert W \right\rvert^3}}{\EE{W^2}^{3/2}} \frac{1}{\sqrt{\lambda}}. $$
Thus, to establish \eqref{eq:berry-esseen}, it only remains to bound $\EE{\left\lvert W \right\rvert^3} / \EE{W^2}^{3/2}$.
Notice that $P_j \eqdef \pi_j w_j^2/\EE{W^2}$ defines a probability distribution on $\{1, \dots, d\}$, and
$$ \frac{\EE{\left\lvert W \right\rvert^3}}{\EE{W^2}} = \EE[P]{|W|} \leq \max_j \{|w_j|\}. $$
Thus,
$$ \frac{\EE{\left\lvert W \right\rvert^3}}{ \EE{W^2}^{3/2}} \leq \sqrt{\frac{\max_j \{w_j^2\}}{\sum_{j = 1}^d \pi_j w_j^2}}. $$
\endproof
\end{lemm}

We are now ready to prove our main result.

\begin{proof}[Proof of Theorem \ref{theo:err}]

The classifier $h$ is a linear classifier of the form
$$ h\p{x} = \I\left\{S> 0\right\} \where S \eqdef \sum_{j = 1}^d w_j x_j , $$
where by assumption $x_j \sim \Pois\p{\lambda_j^{(\tau)}}$.
Our model was fit by dropout, so during training we only get to work with $\tx$ instead of $x$,
where
\begin{align*}
&\tx_j \sim \Binom\p{x_j, \, 1 - \delta}, \text{ and so unconditionally}\\
&\tx_j \sim \Pois\p{\p{1 - \delta} \, \lambda_j^{(\tau)}}.
\end{align*}
Without loss of generality, suppose that $c_\tau = 1$, so that we can write
the error rate $\varepsilon_\tau$ during dropout as
\begin{equation}
\label{eq:boundary}
\varepsilon_\tau = \PP{\tS < 0 \cond \tau}, \where \tS = \sum_{j = 1}^d w_j \tx_j.
\end{equation}
In order to prove our result, we need to translate the information about $\tS$ into information about $S$.

The key to the proof is to show that the sums $S$ and $\tS$ have nearly Gaussian distributions. Let
$$\mu = \sum_{j = 1}^d  \lambda^{(\tau)}_j w_j \eqand \sigma^2 = \sum_{j = 1}^d  \lambda^{(\tau)}_j w_j^2 $$
be the mean and variance of $S$. After dropout,
$$ \EE{\tS} = \p{1 - \delta} {\mu}\eqand \Var{\tS} = \p{1 - \delta} {\sigma^2}. $$
Writing $F_S$ and $F_{\tS}$ for the distributions of $S$ and $\tS$, we see from Lemma \ref{lemm:berry-esseen} that
\begin{align*}
&\sup_{x \in \RR} \left\lvert F_{S}(x) - \Phi\p{\frac{x - \mu}{\sigma}} \right\lvert \leq \CBE \,\sqrt{\Psi_\tau} \, \eqand \\
&\sup_{x \in \RR} \left\lvert F_{\tS}(x) - \Phi\p{\frac{x - \p{1 - \delta} {\mu}}{\sqrt{1 - \delta} \, {\sigma}}} \right\lvert \leq \frac{\CBE}{\sqrt{1 - \delta}} \,\sqrt{\Psi_\tau},
\end{align*}
where $\Psi_\tau$ is as defined in \eqref{eq:error-cut}.
Recall that our objective is to bound $\varepsilon_\tau = F_S(0)$ in terms of $\tvarepsilon_\tau = F_{\tS}(0)$. The above result implies that
\begin{align*}
&\varepsilon_\tau \leq \Phi\p{-\frac{\mu}{\sigma}} + \CBE \, \sqrt{\Psi_\tau} , \eqand \\
&\Phi\p{-\sqrt{1 - \delta} \, \frac{\mu}{\sigma}} \leq \tvarepsilon_\tau + \frac{\CBE}{\sqrt{1 - \delta}} \, \sqrt{\Psi_\tau}.
\end{align*}
Now, writing $t = \sqrt{1 - \delta} \,\mu/\sigma$, we can use the Gaussian tail inequalities
\begin{equation}
\label{eq:gauss_tail}
\frac{\tau}{\tau^2 + 1} < \sqrt{2\pi} \, e^{\frac{\tau^2}{2}} \, \Phi\p{-\tau} < \frac{1}{\tau} \text{ for all } \tau > 0
\end{equation}
to check that for all $t \geq 1$,
\begin{align*}
\Phi\bigg(-&\frac{t}{\sqrt{1 - \delta}}\bigg)
\leq \frac{1}{\sqrt{2\pi}}\,\frac{\sqrt{1 - \delta}}{t} \, e^{-\frac{t^2}{2(1 - \delta)}} \\
&= \frac{\sqrt{1 - \delta} \, t^{\frac{\delta}{1 - \delta}}}{\sqrt{2\pi}^{-\frac{\delta}{1 - \delta}}} \p{\frac{1}{\sqrt{2\pi}} \frac{1}{t} e^{-\frac{t^2}{2}}}^\frac{1}{1 - \delta} \\
&\leq 2^\frac{1}{1 - \delta} \frac{\sqrt{1 - \delta} \, t^{\frac{\delta}{1 - \delta}}}{\sqrt{2\pi}^{-\frac{\delta}{1 - \delta}}} \p{\frac{1}{\sqrt{2\pi}} \frac{t}{t^2 + 1} e^{-\frac{t^2}{2}}}^\frac{1}{1 - \delta} \\
&\leq \frac{2^\frac{1}{1 - \delta} \sqrt{1 - \delta}}{\sqrt{2\pi}^{-\frac{\delta}{1 - \delta}}} \, t^{\frac{\delta}{1 - \delta}} \, \Phi\p{-t}^\frac{1}{1 - \delta}
\end{align*}
and so noting that in $t\,\Phi(-t)$ is monotone decreasing in our range of interest and that $t \leq \sqrt{-2\log \Phi(-t)}$, we conclude that for all $\tvarepsilon_\tau + {\CBE}/{\sqrt{1 - \delta}} \, \sqrt{\Psi_\tau} \leq \Phi(-1)$,
\begin{align}
\label{eq:delta_bound}
\notag
&\varepsilon_\tau \leq  \frac{2^\frac{1}{1 - \delta} \sqrt{1 - \delta}}{\sqrt{4\pi}^{-\frac{\delta}{1 - \delta}}} \p{ \sqrt{-\log \p{ \tvarepsilon + \frac{\CBE}{\sqrt{1 - \delta}} \, \sqrt{\Psi_\tau}}}}^\frac{\delta}{1 - \delta} \\
& \ \ \ \ \ \  \cdot\p{\tvarepsilon + \frac{\CBE}{\sqrt{1 - \delta}} \, \sqrt{\Psi_\tau}}^\frac{1}{1 - \delta}  + \CBE \, \sqrt{\Psi_\tau}.
\end{align}
We can also write the above expression in more condensed form:
\begin{align}
\label{eq:err2}
&\PP{\I\{\hw \cdot x^{(i)}\} \neq c_\tau \cond \taui = \tau} \\
\notag
&\ \ \ \ = \oo \p{\p{\tvarepsilon_\tau + \sqrt{\frac{\max\left\{w_j^2\right\}}{\sum_{j = 1}^d \lambda^{(\tau)}_j w_j^2}}^{(1 - \delta)}}^\frac{1}{1 - \delta} \ \cdot \max\left\{1, \, \sqrt{-\log\p{\tvarepsilon_\tau}}^{\frac{\delta}{1 - \delta}}\right\}}.
\end{align}
The desired conclusion \eqref{eq:error-cut} is equivalent to the above expression, except it uses notation that hides the log factors.
\end{proof}


\begin{proof}[Proof of Theorem \ref{theo:main}]
We can write the dropout error rate as
\begin{align*}
\Err_\delta\p{\hath_\delta} = \Err_\text{min} + \Delta,
\end{align*}
where $\Err_\text{min}$ is the minimal possible error from assumption \eqref{eq:shatter} and $\Delta$ is the the excess error
\begin{align*}
&\Delta = \sum_{\tau = 1}^T  \PP{\tau}  \tvarepsilon_\tau
\cdot \left\lvert \PP{y^{(i)} = 1 \cond \taui = \tau} - \PP{y^{(i)} = 0 \cond \taui = \tau} \right\rvert.
\end{align*}
Here, $\PP{\tau}$ is the probability of observing a document with topic $\tau$ and $\tvarepsilon_\tau$ is as in Theorem \ref{theo:err}.
The equality follows by noting that, for each topic $\tau$, the excess error rate is given by the rate at which we make sub-optimal guesses, i.e., $\tvarepsilon_\tau$, times the excess probability that we make a classification error given that we made a sub-optimal guess, i.e., $\left\lvert \PP{y^{(i)} = 1 \cond \taui = \tau} - \PP{y^{(i)} = 0 \cond \taui = \tau} \right\rvert$.

Now, thanks to \eqref{eq:shatter}, we know that
\begin{align*}
\Err_\delta\p{h^*_\delta} = \Err_\text{min} + \, \oo\p{\frac{1}{\sqrt{\lambda}}},
\end{align*}
and so the generalization error $\teta$ under the dropout measure satisfies
$$ \Delta = \teta + \oo\p{\frac{1}{\sqrt{\lambda}}}. $$
Using \eqref{eq:mintau}, we see that
$$ \tvarepsilon_\tau \leq {\Delta}\big/\p{2\,\alpha\, p_\text{min}} $$
for each $\tau$, and so
$$ \tvarepsilon_\tau = \oo\p{\teta + \frac{1}{\sqrt{\lambda}}} $$
uniformly in $\tau$.
Thus, given the bound \eqref{eq:well-conditioned}, we conclude using \eqref{eq:err2} that
\begin{align*}
&\varepsilon_\tau
=  \oo \p{\p{\teta+ {{\lambda}^{-\frac{1 - \delta}{2}}}}^\frac{1}{1 - \delta} \, \max\left\{1, \, \sqrt{-\log\p{\teta}}^{\frac{\delta}{1 - \delta}}\right\} }
\end{align*}
for each topic $\tau$, and so
\begin{align}
\label{eq:droperr2}
\eta
& =\Err\p{\hath_\delta} - \Err\p{h^*_\delta} \\
\notag
&\ \ =  \oo \p{\p{\teta+ {{\lambda}^{-\frac{1 - \delta}{2}}}}^\frac{1}{1 - \delta}  \, \max\left\{1, \, \sqrt{-\log\p{\teta}}^{\frac{\delta}{1 - \delta}}\right\} },
\end{align}
which directly implies \eqref{eq:droperr}.
Note $\eta$ will in general be larger than the $\varepsilon_\tau$, because guessing the optimal label $c_\tau$ is not guaranteed to lead to a correct classification decision (unless each topic is pure, i.e., only represents one class). Here, substracting the optimal error $\Err\p{h^*_\delta}$ allows us to compensate for this effect.
\end{proof}

\begin{proof}[Proof of Corollary \ref{coro:main}]
Here, we prove the more precise bound
\begin{align}
\label{eq:coro:main2}
\Err\p{\hath_\delta} -  \Err\p{h^*_\delta}
= \oo_P\!\p{\!\sqrt{\p{\frac{d}{n} \!+\!  \frac{1}{{\lambda}^{(1 - \delta)}}}\!\max\left\{1, \, \log\p{\frac{n}{d}}\right\}^{1 + \delta}}^\frac{1}{1 - \delta}}.
\end{align}
To do this, we only need to show that
\begin{align}
\label{eq:gen_rate}
\Err_\delta\p{\hath_\delta} - \Err_\delta\p{h^*_\delta}
 = \oo_P\p{\sqrt{\frac{d}{n}\max\left\{1, \, \log\p{\frac{n}{d}}\right\}}},
\end{align}
i.e., that dropout generalizes at the usual rate with respect to the dropout measure. Then, by applying \eqref{eq:droperr2} from the proof of Theorem \ref{theo:main}, we immediately conclude that $\hath_\delta$ converges at the rate given in \eqref{eq:coro:main} under the data-generating measure.

Let $\widehat{\Err}_\delta(h)$ be the average training loss for a classifier $h$.
The empirical loss is unbiased, i.e.,
$$ \EE{\widehat{\Err}_\delta(h)} = \Err_\delta(h). $$
Given this unbiasedness condition, standard methods for establishing rates as in \eqref{eq:gen_rate} \cite{bousquet2004introduction} only require that the loss due to any single training example $(x^{(i)}, \, y^{(i)})$ is bounded, and that the training examples are independent; these conditions are needed for an application of Hoeffding's inequality. Both of these conditions hold here.
\end{proof}

\subsection{Distinct Topics Assumption}

\begin{prop}
\label{prop:shatter}
Let the generative model from Section \ref{sec:gen} hold, and define
$$\pi^{(\tau)} = \lambda^{(\tau)} / \Norm{\lambda^{(\tau)}}_1$$
for the topic-wise word probability vectors and
$$ \Pi = (\pi^{(1)}, \dots, \pi^{(T)}) \in \R^{d \times T} $$
for the induced matrix. Suppose that $\Pi$ has rank $T$, and that the minimum
singular value of $\Pi$ (in absolute value) is bounded below by
\begin{equation}
\label{eq:shatter_hyp}
\left\lvert\sigma_{\min} \p{\Pi}\right\rvert \geq \sqrt{\frac{T}{(1 - \delta)\lambda}} \, \p{1 + \sqrt{\log_+ \frac{\lambda}{2\pi}}},
\end{equation}
where $\log_+$ is the positive part of $\log$. Then \eqref{eq:shatter} holds.
\end{prop}

\begin{proof}

Our proof has two parts. We begin by showing that, given \eqref{eq:shatter_hyp},
there is a vector $w$ with $\Norm{w}_2 \leq 1$ such that
\begin{align}
\label{eq:margin}
\I\left\{w \cdot \pi^{(\tau)} > 0\right\} = c_\tau, \eqand
\left|
w \cdot \pi^{(\tau)}
\right|
\geq -\frac{1}{\sqrt{(1 - \delta)\lambda}} \, \Phi^{-1}\p{\frac{1}{\sqrt{\lambda}}}
\end{align}
for all topics $\tau$; in other words, the topic centers can be separated with a large margin.
After that, we show that \eqref{eq:margin} implies \eqref{eq:shatter}.

We can re-write the condition \eqref{eq:margin} as
$$ \min\left\{\Norm{w}_2 : c_\tau w \cdot \pi^{(\tau)} \geq 1 \text{ for all } \tau\right\} \leq \p{ -\frac{1}{\sqrt{(1 - \delta)\lambda}} \, \Phi^{-1}\p{\frac{1}{\sqrt{\lambda}}}}^{-1}, $$
or equivalently that
$$ \min\left\{\Norm{w}_2 : S \, \Pi^\top w \geq 1 \right\} \leq \p{ -\frac{1}{\sqrt{(1 - \delta)\lambda}} \, \Phi^{-1}\p{\frac{1}{\sqrt{\lambda}}}}^{-1} $$
where $S = \diag(c_\tau)$ is a diagonal matrix of class signs.
Now, assuming that $\rank(\Pi) \geq T$, we can verify that
\begin{align*}
\min\left\{\Norm{w}_2 :  S \, \Pi^\top w \geq 1 \right\}
&= \min\left\{\sqrt{z^\top \p{\Pi^\top S^2 \Pi}^{-1} z} : z \geq 1 \right\} \\
&\leq \sqrt{1^\top \p{\Pi^\top \Pi}^{-1} 1} \\
&\leq \left\lvert\sigma_{\min} \p{\Pi}\right\rvert^{-1} \sqrt{T} \\
&\leq \p{ \frac{1}{\sqrt{(1 - \delta)\lambda}} \p{1 + \sqrt{\log_+\frac{\lambda}{2\pi}}}}^{-1},
\end{align*}
where the last line followed by hypothesis. Now, by \eqref{eq:gauss_tail}
$$ \Phi\p{-\p{1 + \sqrt{\log_+\frac{\lambda}{2\pi}}}} \leq \frac{1}{\sqrt{2\pi}} \exp\p{-\frac{1}{2} \log\frac{\lambda}{2\pi}} = \frac{1}{\sqrt{\lambda}}. $$
Because $\Phi^{-1}$ is monotone increasing, this implies that
$$ \p{1 + \sqrt{\log_+\frac{\lambda}{2\pi}}}^{-1} \leq \p{-\Phi^{-1}\p{\frac{1}{\sqrt{\lambda}}}}^{-1}, $$
and so \eqref{eq:margin} holds.

Now, taking \eqref{eq:margin} as given, it suffices to check that
the sub-optimal prediction rate is $\oo\p{1/\sqrt{\lambda}}$
uniformly for each $\tau$. Focusing now on a single topic $\tau$, suppose without
loss of generality that $c_\tau = 1$. We thus need to show that
$$ \PP{w \cdot \tx \leq 0} = \oo\p{\frac{1}{\sqrt{\lambda}}}, $$
where $\tx$ is a feature vector thinned by dropout.
By Lemma \ref{lemm:berry-esseen} together with \eqref{eq:well-conditioned}, we
know that
$$ \PP{w \cdot \tx \leq 0} \leq \Phi\p{-\frac{\EE{w \cdot \tx}}{\sqrt{\Var{w \cdot \tx}}}} + \oo\p{\frac{1}{\sqrt{\lambda}}}. $$
By hypothesis,
$$ \EE{w \cdot \tx} \geq - \sqrt{(1 - \delta) \lambda^{(\tau)}} \Phi^{-1}\p{\frac{1}{\sqrt{\lambda}}}, $$
and we can check that
$$ \Var{w \cdot \tx} = (1 - \delta) \sum_{j = 1}^d w_j^2 \lambda^{(\tau)}_j \leq (1 - \delta) \lambda^{(\tau)} $$
because $\Norm{w}_2 \leq 1$. Thus,
$$ \Phi\p{-\frac{\EE{w \cdot \tx}}{\sqrt{\Var{w \cdot \tx}}}} \leq \Phi\p{\Phi^{-1}\p{\frac{1}{\sqrt{\lambda}}}} = \frac{1}{\sqrt{\lambda}}, $$
and \eqref{eq:shatter} holds.
\end{proof}

\subsection{Dropout Preserves the Bayes Decision Boundary}

\begin{proof}[Proof of Proposition \ref{prop:cond}]
Another way to view our topic model is as follows.
For each topic $\tau$, define a distribution over words $\pi^{(\tau)} \in \Delta^{d-1}$:
$\pi^{(\tau)} \eqdef {\lambda^{(\tau)}} / {\|\lambda^{(\tau)}\|_1}.$
The generative model is equivalent to first drawing the length of the document and then
drawing the words from a multinomial:
\begin{align}
\label{eq:multinom_proc}
L_i \sim \Pois\p{\|\lambda^{(\tau)}\|_1}, \eqand
x^{(i)}  \cond \taui, \, L_{i} \sim \Mult\p{\pi^{(\taui)}, \, L_{i}}.
\end{align}

  Now, write the multinomial probability mass function \eqref{eq:multinom_proc} as
  \begin{equation*}
    \PP[m]{x;\, \pi, \, L} = \frac{L!}{x_1!\cdots x_p!} \, \pi_1^{x_1} \cdots \pi_d^{x_d}
  \end{equation*}
For each label $c$, define $\Pi_c$ to be the distribution over the probability vectors
induced by the distribution over topics.  Note that we could have an infinite
number of topics.
  By Bayes rule, 
  \begin{align*}
    &\PP{x=v\cond y = c}  = \PP{L = \sum_{j = 1}^d v_j} \cdot \int \PP[m]{v;\, \pi , \, \sum_{j = 1}^d v_j}\,d\Pi_c(\pi), \eqand \\
    &\PP{y = c\cond x = v} = \frac{\PP{c} \int \PP[m]{v; \, \pi, \sum_{j = 1}^d v_j} \, d\Pi_c(\pi)}
       {\sum_{c'}\PP{c'}\int \PP[m]{v; \, \pi, \sum_{j = 1}^d v_j} \, d\Pi_{c'}(\pi)}.
  \end{align*}
  The key part is that the distribution of $L$ doesn't depend on $c$,
  so that when we condition on $x = v$, it cancels.
  As for the joint distribution of $(\widetilde x, y)$, note that,
  given $\pi$ and $\tL = \sum_{j = 1}^d \tx_j$, $\tx$
  is conditionally $\Mult(\pi,\tL)$.  So then
  \begin{align*}
    &\PP{\tx=v\cond y = c} = \PP{\tL = \sum_{j = 1}^d v_j} \cdot \int \PP[m]{v;\, \pi , \, \sum_{j = 1}^d v_j}\,d\Pi_c(\pi), \eqand \\
    &\PP{y = c\cond \tx = v}  = \frac{\PP{c} \int \PP[m]{v; \, \pi, \sum_{j = 1}^d v_j} \, d\Pi_c(\pi)}
       {\sum_{c'}\PP{c'}\int \PP[m]{v; \, \pi, \sum_{j = 1}^d v_j} \, d\Pi_{c'}(\pi)}.
  \end{align*}
  In both cases, $L$ and $\tilde L$ don't depend on the topic,
  and when we condition on $x$ and $\tx$, we get the same distribution over $y$.
\end{proof}

\end{appendix}

\end{document}